\documentclass[11pt, twoside]{amsart}

\usepackage[utf8]{inputenc}
\usepackage[T1]{fontenc}
\usepackage{amssymb,amsmath,amstext}
\usepackage{hyperref, graphicx}
\usepackage{comment}
\usepackage{mathtools}
\usepackage{amsaddr}

\newtheorem{theor}{Theorem}[section]
\newtheorem{lem}[theor]{Lemma}
\newtheorem{defin}[theor]{Definition}

\newtheorem{exam}[theor]{Example}

\newtheorem{cor}[theor]{Corollary}
\newtheorem{rem}[theor]{Remark}

\newtheorem{claim}[theor]{Claim}

\numberwithin{equation}{section}

\newcommand{\mr}{\mathrm}

\newcommand{\es}{\emptyset}

\newcommand{\nts}{\negthickspace}
\newcommand{\uhrc}{\nts \upharpoonright \nts}

\newcommand{\mcA}{\mathcal{A}}
\newcommand{\mcB}{\mathcal{B}}

\newcommand{\mfS}{\mathfrak{S}}

\newcommand{\mfm}{\mathfrak{m}}

\newcommand{\mbS}{\mathbf{S}}

\newcommand{\mbW}{\mathbf{W}}
\newcommand{\mbX}{\mathbf{X}}

\newcommand{\mbbP}{\mathbb{P}}
\newcommand{\mbbN}{\mathbb{N}}
\newcommand{\mbbR}{\mathbb{R}}

\title[A concentration result for neural networks]{A concentration result for multilayer feedforward neural networks}

\author{Vera Koponen}

\address{Department of Mathematics, Uppsala University, Sweden.}
\email{vera.koponen@math.uu.se}

\date{3 December, 2026}

\begin{document}

\begin{abstract}
We consider for an arbitrary fixed $\rho$ and for each positive integer $n$ a multilayer 
feedforward artificial neural network with $\rho$ layers,
$n$ neurons in the
first layer (the input layer) 
and only one neuron, the output neuron, in the last layer.
Very roughly formulated, the main result is that if the distribution of
weights of connections from a layer to the next are, for all large $n$,  approximated well by a fixed continuous 
(but otherwise arbitrary) curve which does not depend on $n$, 
and if the values of the $n$ input neurons are independently and identically distributed
with a continuous probability density function,
then there is a number $\psi$ such that for all $\varepsilon > 0$
the probability that the value of the output neuron is in $[\psi - \varepsilon, \psi + \varepsilon]$ tends to 1
as $n$ tends to infinity.
\end{abstract}

\maketitle

\section{Introduction}

\noindent
Artificial neural networks (NNs) \cite{Good, Meh} are widely used in many different contexts and often successfull in
predicting outcomes or classifying objects (e.g. images).
But they do not produce any explanations, or justifications, of their predictions, or classifications.
This raises questions about their reliability, interpretability, fairness, and safety, 
and it motivates efforts to understand the internal 
behaviour of NNs (see e.g. \cite{Cai, Gui, Xu, Zha}).

Although the question of interpretability of NNs have been raised and considered before,
as discussed in e.g. \cite{Cai, Gui, Xu, Zha}, there appears to be very little theoretical work 
on the question of how one can (with high probability) predict and explain the output of a large NN if one is given only
information about its general architecture and its input (vector).
In this study we explore, under admittedly strong assumptions, a potentially fruitful theoretical methodology 
for understanding the mentioned type of question of predicting and explaining the output of a large NN.
No claims are made that the presented results will have a decidable impact on the practical use and understanding
of NNs. But I do believe that the methodology introduced, and proved to work in the context of his study,
has the potential of generalization and of becoming a useful
tool for understanding the behaviour of NNs in more realistic settings.

We will consider the most basic type of NN, the multilayer feedforward NN in which every neuron 
is assigned a real number, its (activation) value, or signal strength, once all of its input neurons
have been assigned values. We also assume that the NN has only one output neuron, and that every
neuron, except the input neurons, is connected to all neurons in the previous layer.
In this setting 
three things determine the value of the output neuron.
Firstly, the values of the input neurons.
Secondly, the weights (which are real numbers) of the connections from one layer to the next.
Thirdly, the so-called activation function which is a specification of how 
to assign a value to a (not input) neuron based on the values of neurons in the previous
layer and the weights of connections from neurons in the previous layer.

We wish to understand how the value of the output neuron correlates with the properties of the NN
such as: the number of layers, the number of neurons in each layer, the distribution of weights of connections from
one layer to the next, and the values of the input neurons.
For a specific NN the value out the output neuron can be exactly computed given 
complete information about the values of the input neurons and complete information about the NN itself,
including knowing the weight of each individual connection.
But here we are interested in finding more general correlations for large NNs. 
The focus on large NNs is motivated by the fact that such are common and that they are harder to analyse
as many weighted connections are involved in producing the output of the NN.
That is, we wish to be able to make conclusions about the distribution of the values of the output neuron
given only partial information about the (large) NN. 
In this study such partial information will be
the number of layers, some very rough bounds on the number of neurons in each layer,
and information about the approximate distribution of weights.
In addition we will assume that the activation functions of neurons behave in a continuous way.
Since we will consider NNs with increasing number of neurons (tending to infinity) we need to formulate
the activation functions in such a way that the notion of ``continuity'' used makes sense for functions
that take any finite sequence (of any finite length) of reals as input and return a real as output.
This will be done by expressing the activation functions in terms of continuous aggregation functions
as defined in the beginning of Section~\ref{Aggregation functions}.

When using an NN we probably do not know the exact values assigned to the input neurons.
But we may have some knowledge, or at least an educated guess, about the probabilistic distribution of values 
assigned to the input neurons.
So given partial information about an NN (as explained above),
a probability distribution on the set of possible combinations of values to the input neurons
and an interval $I$, 
we can ask what the probability is that the value of the output neuron belongs to $I$.

In this article we prove a ``concentration result'' for the output neuron, formally stated as
Theorem~\ref{main theorem technical statement} and informally stated below.
We assume throughout that for an arbitrary constant $\mfm \geq 1$ the values of neurons and weights of connections
are always in the interval $[-\mfm, \mfm]$.

\medskip

\noindent
{\bf Main result, in a simplified formulation.} {\em 
Suppose that for every large enough integer $n$,
$\mcB_n$ is an NN with $\rho \geq 2$ layers, $n$ neurons in the first layer, the input layer, 
and with a single output neuron in layer $\rho$.
For $i = 2, \ldots, \rho$, let
$W_i : [-\mfm, \mfm] \to [0, \infty)$ be continuous with$\int_{[-\mfm, \mfm]} W_i(x) dx = 1$, and suppose that
the following assumptions hold:
\begin{enumerate}
\item For every interval $I \subseteq [-\mfm, \mfm]$ 
and $\delta > 0$, if $b$ is a neuron of $\mcB_n$ in layer $i \geq 2$ and $n$ is large enough,
then the proportion of connections to $b$ with weight in $I$ from neurons in the previous layer 
is between $\int_I W_i(x) dx - \delta$ and $\int_I W_i(x) dx + \delta$.

\item For all $i = 2, \ldots, \rho$, the activation function for neurons in layer $i$ is defined by using a
continuous aggregation function $F_i$ (independent of $n$), 
in the sense of 
Definition~\ref{definition of continuous aggregation function}
below, and a continuous function $g_i : [-\mfm, \mfm]^2 \to [-\mfm, \mfm]$, as
explained in~(A8) below.
\end{enumerate}
If the values of the neurons of the input layer are independently and identically distributed
with a continuous probability density function,
then there is $\psi \in [-\mfm, \mfm]$ such that for all $\varepsilon > 0$,
the probability that the value of the output neuron 
is in $[\psi - \varepsilon, \psi + \varepsilon]$ tends to 1 as $n \to \infty$.
The number $\psi$ can be computed by only using the functions $W_i, g_i$ and aggregation function $F_i$,
for $i = 2, \ldots, \rho$, and the (identical) probability density function for each neuron.
}

\medskip

\noindent
The proof of the main result has as a corollary that if $\mcB_n$ (for all sufficently large $n$)
and the distribution of values of input neurons are as assumed above,
then there is and NN $\mcB'_n$  (for all sufficiently large $n$) with only 2 layers and
with high probability the value of the output neuron of $\mcB'_n$
is in $[\psi - \varepsilon, \psi - \varepsilon]$ for the same $\psi$ as above and all sufficiently large $n$.
This is stated precisely by Corollary~\ref{two layers suffice}.

The probably most restrictive assumption made in the above theorem is that the values of the input neurons
are independently and identically distributed, which can be interpreted as the input being random noise.
A natural direction of further research, towards more realistic scenarios,
is to investigate distributions on the sequences (vectors) 
of values of the input neurons that do not assume that the values of the individual input neurons 
are independent of each other.

If there are arbitrarily large $n$ and $m$ such that the weights from, say from layer 1 to layer 2,
in $\mcB_n$, respectively $\mcB_m$, are distributed too differently then we cannot hope that
the distribution of the values of the output neuron tends to a fixed distribution as $n \to \infty$.
This is illustrated by Example~\ref{necessity of uniform W}.
Therefore some sort of uniformity condition is needed which stipulates that the distribution of
weights from, say layer 1 to layer 2, does not depend too much on $n$.

One can find out the distribution of weights of connections to any neuron from neurons in the previous layer
with relatively small computational resources:
Suppose that the NN has $\rho$ layers (where $\rho$ is fixed) and at most $n$ neurons in each layer.
Given any non input neuron we can sort, by size, the weights of the at most $n$ incomming connections in 
time $O(n \log n)$. There are at most $(\rho - 1)n$ non input neurons so the time needed for repeating this
for all non input neurons is bounded by $O(n^2 \log n)$. 
After sorting the weights of the incomming connections to a neuron $b$ we can inspect if their distribution follows,
up to a small deviation, a curve 
(represented by $W_i$ if $b$ belongs to the $i^{th}$ layer)
which does not depend on the size of layer $i-1$.
Thus checking if condition~(1) holds is computationally feasible.

Remark~\ref{remark on generalizations}
explains how the main theorem can straightforwardly be generalized to a context where each neuron 
has several {\em features} where each feature has its own value.

\subsection*{Related work and further research}

Convergence properties of graph neural networks (GNNs) have been studied by Adam-Day et. al.
in \cite{Adam-Day1, Adam-Day2} for certain probability distributions on the underlying graph.
However the GNNs of \cite{Adam-Day1, Adam-Day2} do not consider weights of connections from one
layer to the next. Once the random graph is chosen in \cite{Adam-Day1, Adam-Day2}
the values of the node features are updated a certain number of steps by means of aggregating (locally or globally) over the values
of node features in the previous step and using a Lipschitz continuous updating function.

From a technical point of view the most closely related work is probably \cite{KopCont}.
This study is also inspired by studies on convergence phenomena such as
\cite{CM, Jae98a, Kop20, Kop26, Kop26pre, KopMLN, KT1, KW1, KW2, TK1, Wei21, Wei24}
 in the field of 
Statistical Relational Artificial Intelligence \cite{BKNP, RKNP, GT, KMG}.

As already mentioned, it would be interesting to extend the theorem above to other settings.
The most important directions are probably (a) to consider other types of distributions on the vector of input neurons that
are more realistic in practice, and (b) to consider more flexible variations of the weights, for example allowing
different neurons in the same layer to have different distributions of the weights of the incomming connections.
But some degree of uniformity on the weights is necessary as pointed out above and illustrated in
Example~\ref{necessity of uniform W}.
A topic related to (a) is to consider structured inputs, for example a 2-dimensional grid of neurons 
which could represent an image, and a probability distribution on such grids.
The model theoretic framework used here can represent NNs with such structured inputs as well.

\subsection*{Preliminaries}

By $\mbbN^+$ we denote the set of positive integers.
For a positive integer $n$, $[n] := \{1, \ldots, n\}$.
If $S$ is a set then $|S|$ is its cardinality which we also call size if $S$ is finite.
Finite sequence may be denoted by $\bar{s}$, for some symbol $s$, and then $|\bar{s}|$ denotes the length of the sequence.

In order to formalize the notions of a multilayer feedforward neural network (NN) and its states
the concept of continuous, or general, structure (see e.g. \cite{BenY, CK}) will be used.
This concept generalizes the concept of ``discrete'' structure (see e.g. \cite{EF, Gra, Hod})
which includes for example graphs in the usal sense of discrete mathematics.
But besides the basic definitions given below no further knowledge of model theory is necessary to understand the
results and proofs.

By a {\em signature} we mean a finite set of so-called {\em relation symbols}, each with an associated {\em arity}
which is a positive integer.
By a {\em continuous structure}, or just {\em structure}, say $\mcA$,
we mean a nonempty set $A$ called the {\em domain} of $\mcA$ combined with an {\em interpretation} of 
every  $R \in \sigma$ which, if $R$ has arity $k$, is a function from $A^k$ to the set of real numbers $\mbbR$;
this function, the interpretation of $R$ in $\mcA$, is denoted by $R^\mcA$.
In this study the range of $R^\mcA$ will always be included in $[-\mfm, \mfm]$ for some fixed, but arbitrary, $\mfm \geq 1$.

The following is a direct consequence of \cite[Corollary~A.1.14]{AS} which in turn follows from a bound given by
Chernoff \cite{Che}:

\begin{lem}\label{independent bernoulli trials}
Let $Z$ be the sum of $n$ independent 0/1-valued random variables, each one with probability $p$ of having the value 1,
where $p > 0$.
For every $\varepsilon > 0$ there is $c_\varepsilon > 0$, depending only on $\varepsilon$, such that the probability that
$|Z - pn| > \varepsilon p n$ is less than $2 e^{-c_\varepsilon p n}$.
(If $p = 0$ then the same statement holds if `$2 e^{-c_\varepsilon p n}$' is replaced by `$e^{-n}$'.)
\end{lem}

\section{Aggregation functions}\label{Aggregation functions}

\noindent
In order to define activation functions for neurons in a way that does not depend on the number
of neurons of an NN we will use the concept of aggregation function.
We fix some real $\mfm \geq 1$.

\begin{defin}{\rm
(a) Let $[-\mfm, \mfm]^{fin}$ denote the set of all finite nonempty sequences of reals from the interval 
$[-\mfm, \mfm]$, or equivalently,
$[-\mfm, \mfm]^{fin} = \bigcup_{k = 1}^\infty [-\mfm, \mfm]^k$.\\
(b) A function $F : [-\mfm, \mfm]^{fin} \to [-\mfm, \mfm]$ will be called an {\em aggregation function} if it is symmetric in the sense that
if $(x_1, \ldots, x_k) \in [-\mfm, \mfm]^k$ and $(y_1, \ldots, y_k)$ is a permutation (reordering) of $(x_1, \ldots, x_k)$ then
$F(x_1, \ldots, x_k) = F(y_1, \ldots, y_k)$.
}\end{defin}

\noindent
Examples of aggregation functions include the average of a sequence,
\[
\mr{av}(x_1, \ldots, x_n) := \frac{x_1 + \ldots + x_n}{n}
\]
and the maximum of a sequence,
\[
\mr{max}(x_1, \ldots, x_n) := \text{ the largest number among $x_1, \ldots, x_n$.}
\]
We will require that aggregation functions used by NNs are continuous in the sense defined below.
The intuition behind the technical second part of the definition is that if $M \in \mbbN^+$ is large enough,
$n, m \geq M$,
$\bar{q} = (q_1, \ldots, q_n) \in [-\mfm, \mfm]^n$, $\bar{q}' = (q'_1, \ldots, q'_m) \in [-\mfm, \mfm]^m$
and, for every interval $I \subseteq [-\mfm, \mfm]$ of lengh $M$, the proportion of entries in $\bar{q}$
that belong to $I$ is roughly the same as the proportion of entries in $\bar{q}'$ that belong to $I$,
then $F(\bar{q})$ and $F(\bar{q}')$ are roughly the same.

\begin{defin}\label{definition of continuous aggregation function}{\rm
Let $F : [-\mfm, \mfm]^{fin} \to [-\mfm, \mfm]$ be an aggregation function.
We call $F$ {\em continuous} if the following two conditions hold:
\begin{enumerate}
\item For every $\varepsilon > 0$ there is $\delta > 0$ such that for all $n$,
if $(q_1, \ldots, q_n)$, $(q'_1, \ldots, q'_n)$ $\in [-\mfm, \mfm]^n$ and $|q_i - q'_i| \leq \delta$ for all $i = 1, \dots, n$,
then $|F(q_1, \ldots, q_n) - F(q'_1, \ldots, q'_n)| \leq \varepsilon$.

\item For every $\varepsilon > 0$ there are $\delta > 0$ and $M, N \in \mbbN^+$
such that if $\alpha_0, \ldots, \alpha_{M-1} \in [0, 1]$ and 
$(q_1, \ldots, q_n), (q'_1, \ldots, q'_m) \in [-\mfm, \mfm]^{fin}$ (where we may have $n \neq m$) are such that
conditions (a)--(d) below hold, then $|F(q_1, \ldots, q_n) - F(q'_1, \ldots, q'_m)| \leq \varepsilon$:
\begin{enumerate}
\item $n, m \geq N$, 

\item for all integers $i$ such that $-M \leq i < M$, if $\alpha_i > 0$ then $\alpha_i > \delta$,
and (regardless of whether $\alpha_i > 0$ or not)
\begin{align*}
&\frac{\big|\big\{l : q_l \in \big[\frac{i\mfm}{M}, \frac{(i+1)\mfm}{M}\big] \big\}\big|}{n} \in (\alpha_i - \delta, \alpha_i + \delta), \\
&\frac{\big|\big\{l : q'_l \in \big[\frac{i\mfm}{M}, \frac{(i+1)\mfm}{M}\big] \big\}\big|}{m} \in (\alpha_i - \delta, \alpha_i + \delta), 
\end{align*}

\item if either $-M \leq i < M-2$ and $\alpha_i = \alpha_{i+1} = \alpha_{i+2} = 0$, 
or $i = M-2$ and $\alpha_i = \alpha_{i+1} = 0$,
then
\begin{align*}
\big\{l : q_l \in \big[ {\scriptstyle \frac{(i+1)\mfm}{M}, \frac{(i+2)\mfm}{M} } \big] \big\} = 
\big\{l : q'_l \in \big[ {\scriptstyle\frac{(i+1)\mfm}{M}, \frac{(i+2)\mfm}{M} } \big] \big\} = \es, \ \text{ and}
\end{align*}

\item if $\alpha_{-M} = \alpha_{-M+1} = 0$, then
\[
\big\{l : q_l \in \big[ {\scriptstyle -\mfm, \frac{(-M+1)\mfm}{M} } \big] \big\} = 
\big\{l : q'_l \in \big[ {\scriptstyle -\mfm, \frac{(-M+1)\mfm}{M} } \big] \big\} = \es.
\]
\end{enumerate}
\end{enumerate}

}\end{defin}

\noindent
In \cite[Definition~2.4]{KopCont} aggregation functions of finite sequences of reals in $[0, 1]$ are considered.
Besides that we consider a larger interval $[-\mfm, \mfm]$ here (if $\mfm \geq 1$), 
Definition~\ref{definition of continuous aggregation function} above is, in theory at least, slightly stronger 
than the definition of continuous aggregation function in \cite[Definition~2.4]{KopCont}
since conditions~(c) and~(d) in \cite[Definition~2.4]{KopCont} apply to possibly fewer sequences than
conditions~(c) and~(d) above.

\begin{lem}
The aggregation functions maximum, minimum and average are continuous.
\end{lem}

\noindent
{\bf Proof.}
The proof is essentially like the proof of the corresponding statement in \cite[Lemma~2.5]{KopCont},
but for completeness we give it here.
We first consider maximum.
As when $\max$ is restricted, for any $n$, to $[-\mfm, \mfm]^n$ then $\max$ is continuous in the usual sense
and hence condition~(1) of Definition~\ref{definition of continuous aggregation function} is satisfied.
We now verify that condition~(2) of the same definition holds.
Let $\varepsilon > 0$.
Suppose that 
conditions (a)--(c) of part~(2) of 
Definition~\ref{definition of continuous aggregation function}
are satisfied by $\delta > 0$, $M, N \in \mbbN^+$,
$\alpha_0, \ldots, \alpha_{M-1} \in [0, 1]$, and 
$(q_1, \ldots, q_n), (q'_1, \ldots, q'_m)$.
For $-M \leq i < M$, let $P_i = \big\{l : q_l \in \big[\frac{i\mfm}{M}, \frac{(i+1)\mfm}{M}\big] \big\}$
and  $P'_i = \big\{l : q'_l \in \big[\frac{i\mfm}{M}, \frac{(i+1)\mfm}{M}\big] \big\}$.
Let $s$ be maximal such that $P_s \neq \es$.
Then 
\[
\frac{s\mfm}{M} \leq \max(\bar{q}) \leq \frac{(s+1)\mfm}{M}.
\]
By the maximality of $s$ and Assumption (2)(b) of Definition~\ref{definition of continuous aggregation function}
we must have $\alpha_i = 0$ for all $i = s+1, \ldots, M$.
Assuming that $\delta > 0$ is small enough we have $\alpha_0 + \ldots + \alpha_{M-1} \approx 1$ so there is 
$i \leq s$ such that $\alpha_i > 0$. Let $t$ be maximal such that $\alpha_t > 0$. Then $t \leq s$.
By~(2)(c) we have $P'_i = \es$ for all $i > s+1$ and therefore 
\[
\max(\bar{q}') \leq \frac{(s+2)\mfm}{M}.
\]
Suppose, for a contradiction, $P'_{s-1} = P'_s = P'_{s+1} = \es$ (or $P'_{s-1} = P'_s = \es$ if $s = M-1$).
From~2(b) of 
Definition~\ref{definition of continuous aggregation function}
we get $\alpha_{s-1} = \alpha_s = \alpha_{s+1} = 0$.
By~2(c) of the same definition we must then have $P_s = \es$, contradicting the choice of $s$.
We conclude that at least one of $P'_{s-1}$, $P'_s$, and $P'_{s+1}$ must be nonempty and therefore
\[
\max(\bar{q}') \geq \frac{(s-1)\mfm}{M}.
\]
The above bounds of $\max(\bar{q})$ and $\max(\bar{q}')$ imply that
$|\max(\bar{q}) - \max(\bar{q}')| \leq \tfrac{3\mfm}{M}$.
So if $\delta > 0$ is chosen small enough, $M$ and $N$ large enough,
and conditions~(2)(a)--(2)(d) hold, then 
 $|\max(\bar{q}) - \max(\bar{q}')| \leq \varepsilon$.

The proof for minimum is analogous but here we also use condition~2(d)
of Definition~\ref{definition of continuous aggregation function}.
So we now consider the average, or am.
It is again clear that am has property~(1) of Definition~\ref{definition of continuous aggregation function},
so we only verify property~(2).
Suppose that $\bar{q} = (q_1, \ldots, q_n)$ and $\bar{q}' = (q'_1, \ldots, q'_m)$ 
satisfy conditions (a)--(d) of part~(2) of 
Definition~\ref{definition of continuous aggregation function}.
Then 
\[
\sum_{i=0}^{M-1} \frac{i\mfm}{M}(\alpha_i - \delta) \leq \mr{am}(\bar{q}) \leq 
\sum_{i=0}^{M-1} \frac{(i+1)\mfm}{M}(\alpha_i + \delta)
\]
and similarly if $\bar{q}$ is replaced by $\bar{q}'$.
Hence $|\mr{am}(\bar{q}) - \mr{am}(\bar{q}')| \leq 
\sum_{i=0}^{M-1} \frac{(\alpha_i + 2i\delta + \delta)\mfm}{M}$, and if $\delta$ is small enough 
(relative to $M$) then this is approximately $\sum_{i=0}^{M-1} \frac{\alpha_i \mfm}{M}$.
For sufficiently small $\delta$, the sum $\alpha_0 + \ldots + \alpha_{M-1}$ is approximately 1,
so $\sum_{i=0}^{M-1} \frac{\alpha_i \mfm}{M} \approx \frac{\mfm}{M}$ which can be made as small as we like if $M$ is chosen
large enough.
\hfill $\square$

\medskip
\noindent
The following lemma is a direct consequence of the definition of continuous aggregation function, 
so we omit its proof.

\begin{lem}\label{composition of aggregation function and continuous function}
If $F : [-\mfm, \mfm]^{fin} \to [-\mfm, \mfm]$ is a continuous aggregation function and $h : [-\mfm, \mfm] \to [-\mfm, \mfm]$ is 
continuous, then $G : [-\mfm, \mfm]^{fin} \to [-\mfm, \mfm]$ defined by 
$G(q_1 \ldots, q_n) := h(F(q_1, \ldots, q_n))$ is a continuous aggregation function.
\end{lem}

\section{Model theoretic representation}\label{Model theoretic representation}

\noindent
We fix some $\rho \geq 2$ and 
consider a sequence of neural networks (NNs), denoted $\mcB_n$ (`$\mcB$' for ``base structure'') for $n \in \mbbN^+$, 
such that each $\mcB_n$ has exactly
$\rho$ layers of neurons and the
first layer (the input layer) of $\mcB_n$ has exactly $n$ neurons, so its width is $n$.
For each $n$ and $i = 1, \ldots, \rho$, the $i^{th}$ layer of $\mcB_n$ is denoted by $L_{n, i}$.
We also assume that the last layer, $L_{n, \rho}$, is a singleton. We let $L_{n, \rho} := \{o_n\}$
and call $o_n$ the {\em output neuron} of $\mcB_n$.

Fix some real number $\mfm \geq 1$, where we think of $\mfm$ as the maximal absolute value that neurons or connections
betwen them can have.
For all $n$, $i = 1, \ldots, \rho - 1$ and all choices of $a \in L_{n, i}$ and $b \in L_{n, i+1}$ there is a
{\em connection} from $a$ to $b$ which carries a {\em weight} which is a real number in $[-\mfm, \mfm]$.
Once values are assigned to all input neurons of $\mcB_n$ (the neurons in $L_{n, 1}$)
all other neurons should automatically be assigned values that depend only on the values of the input neurons,
the weights of the connections, and some activation function. 
The collection of values of all neurons obtained in this way will be thought of as a {\em state} of $\mcB_n$,
and we are interested in the set of all states that $\mcB_n$ can be in.
To make these ideas mathematically precise we formulate them in a model theoretic language.
Since values of neurons and weights of connections are real numbers we will consider {\em continuous structures}
in which relation symbols are interpreted as functions taking values in $[-\mfm, \mfm]$.
So if $\sigma$ is a signature (i.e. set of relation symbols), $\mcA$ is a continuous $\sigma$-structure, and $R \in \sigma$
has arity $k$, then the {\em interpretation of $R$ in $\mcA$}, denoted $R^\mcA$, is a function
from $A^k$ (where $A$ is the domain of $\mcA$) to $\mbbR$. 
Since we have, in addition, stipulated that all values must be in $[-\mfm, \mfm]$ we also
require $R^\mcA$ to be a mapping into $[-\mfm, \mfm]$.

Let $\tau = \{L_1, \ldots, L_\rho, E\}$ where $L_1, \ldots, L_\rho$
have arity $\rho$ and $E$ has arity 2.
For each $n$, $\mcB_n$ is a $\tau$-structure with domain $B_n$ such that

\begin{enumerate}
\item[(A1)] For all $b \in B_n$, there is exactly one $i \in [\rho]$ such that $L_i^{\mcB_n}(b) = 1$, and if $j \neq i$ then
$L_j^{\mcB_n}(b) = 0$. This expresses that each neuron belongs to exactly one layer, and we
define 
\[
L_{n, i} := \{b \in B_n : L_i^{\mcB_n}(b) = 1\}.
\]

\item[(A2)] If $i, j \in [\rho]$, $j \leq i$ or $j-i > 1$, then for all $a \in L_{n, i}$ and $b \in L_{n, j}$ we have
$E^{\mcB_n}(a, b) = 0$. This expresses that a connection with positive weight must go from some layer
$L_{n, i}$ to the next layer $L_{n, i+1}$.

\end{enumerate}

\noindent
We will assume that all $\mcB_n$, $n \in \mbbN^+$, have ``{\em similar design}'' which besides having the same
number of layers will mean that they use the same activation functions 
(here expressed by aggregation functions) and that the
weights of connections between two fixed consecutive layers have similar distributions for all sufficiently large $n$.
First we fix some functions that will specify a ``similarity class'' of sequences of NNs.

\begin{enumerate}
\item[(A3)] Let $\zeta(x)$ be a polynomial such that $\zeta(x) > 0$ for $x > 0$ and $\lim_{x\to\infty}\zeta(x) = \infty$.

\item[(A4)] Let $W_2, \ldots, W_\rho : [-\mfm, \mfm] \to [0, \infty)$ and $g_2, \ldots, g_\rho : [-\mfm, \mfm]^2 \to [-\mfm, \mfm]$
be continuous functions such that $\int_{[-\mfm, \mfm]} W_i(x) dx = 1$ for all $i = 2, \ldots, \rho$.

\item[(A5)] Let $F_2, \ldots, F_\rho : [-\mfm, \mfm]^{fin} \to [-\mfm, \mfm]$ be continuous aggregation functions 
in the sense of Defintion~\ref{definition of continuous aggregation function}.

\end{enumerate}

\noindent
Next, we assume that the following conditions hold:

\begin{enumerate}
\item[(A6)] $|L_{n, 1}| = n$, $|L_{n, \rho}| = 1$, and for all $i = 1, \ldots, \rho$, $0 < |L_{n, i}| \leq \zeta(n)$, 
so layer $i$ of $\mcB_n$ has at most $\zeta(n)$
neurons, that is, its width is at most $\zeta(n)$. 

\item[(A7)] For all $i = 2, \ldots, \rho$, for every interval $I \subseteq [-\mfm, \mfm]$ with different endpoints, 
and for every $\varepsilon > 0$,
if $n$ is sufficiently large and $b \in L_{n, i}$, then 
\[
\int_I W_i(x) dx - \varepsilon \leq 
\frac{\big|\big\{ a \in L_{n, i-1} : E^{\mcB_n}(a, b) \in I \big\}\big|}{\big|L_{n, i}\big|}
\leq \int_I W_i(x) dx + \varepsilon,
\]
and if $\int_I W_i(x) dx = 0$, then the proportion above is 0.
Informally, this means that the distribution of weights of connections from 
$L_{n, i-1}$ to $b \in L_{n, i}$ is approximated well by $W_i$
for large $n$. 

\item[(A8)] For $i = 2, \ldots, \rho$ and any $b \in L_{n, i}$, if 
$t(n, i-1)$ denotes the number of neurons in layer $i-1$, 
$v_j$ ($j = 1, \ldots, t(n, i)$) is the value of the $j^{th}$ neuron in $L_{n, i-1}$ (in some arbitrary order), and $w_j$ is the 
weight of the connection from the $j^{th}$ neuron in $L_{n, i-1}$ to $b$, then the value of $b$ is assigned to be
\begin{equation*}\label{specification of the activation function}
F_i\big(g_i(v_1, w_1), \ldots, g_i(v_{t(n, i-1)}, w_{t(n, i-1)})\big).
\end{equation*}
Hence the above expression defines the activation function for neurons in layer $i$ of $\mcB_n$, for all $n$.

\end{enumerate}

\begin{rem}\label{remark on the number of neuros in each layer}{\rm
Fix any $i \in \{2, \ldots, \rho\}$.
Since $\int_{[-\mfm, \mfm]} W_i(x) dx = 1$ and $W_i$ is continuous on $[-\mfm, \mfm]$, 
hence uniformly continuous on the same interval, there are 
$\alpha, \beta \in [-\mfm, \mfm]$ such that $\alpha < \beta$ and $W_i(x) > 0$ for all $x \in [\alpha, \beta]$.
Let $M \in \mbbN^+$.
It follows that if,
for all $j = 0, \ldots, M-1$, we let 
$I_j := \big[ \alpha + \frac{j(\beta - \alpha)}{M}, \ \alpha + \frac{(j+1)(\beta - \alpha)}{M}\big]$,
then $\int_{I_j} W_i(x) dx > 0$.
It now follows from Assumption~(A7) that if $n$ is sufficiently large, then
$|L_{n, i-1}| \geq M$.
Since $M$ was arbitrary it follows that, as $n \to \infty$, then number of neurons in layer $i-1$ tends to infinity.
So the width of each layer except for the output layer tends to infinity as $n \to \infty$.
}
\end{rem}

\noindent
We assumed that the input layer $L_{n, 1}$ of $\mcB_n$ has exactly $n$ neurons.
Without loss of generality we may (to simplify notation) assume that $L_{n, 1} = [n] := \{1, \ldots, n\}$.

We now formally define the notion of state of $\mcB_n$, the state space of $\mcB_n$,
and the probability measure on the state space.
Let $\sigma := \tau \cup \{P\}$ where $P$ has arity 1. 
The value of a neuron $a$ of $\mcB_n$
in a particular state will be represented by the value of $P(a)$ in an expansion to $\sigma$ of $\mcB_n$.
More precisely, let $\mbS_n$ consist of all $\sigma$-structures $\mcA$ such that
$\mcA$ is an expansion of $\mcB_n$ and
\begin{enumerate}
\item[(A9)] For all $i = 2, \ldots, \rho$ and all $b \in L_{n, i}$, if 
$a_1, \ldots, a_{t(n, i-1)}$ is a list without repetitions of all neurons in $L_{n, i-1}$, in any order, then
\begin{equation*}\label{how a neuron gets its value}
P^\mcA(b) = F_i\big(g_i\big(P^\mcA(a_1), w_1\big), \ldots, g_i\big(P^\mcA(a_{t(n, i-1)}), w_{t(n, i-1)}\big)\big).
\end{equation*}
\end{enumerate}

\noindent
For each $(x_1, \ldots, x_n) \in [-\mfm, \mfm]^n$ let 
$\mfS_n(x_1, \ldots, x_n)$ be the unique $\mcA \in \mbS_n$ such that 
$x_k = P^\mcA(k)$ for all $k \in L_{n, 1} := [n]$.
By~(A9), $\mfS_n$ is a bijection.

Let $\mu : [-\mfm, \mfm] \to [0, \infty)$ be a probability density function.
Define $\mu_n : [-\mfm, \mfm]^n \to [0, \infty)$ by $\mu_n(x_1, \ldots, x_n) = \mu(x_1) \cdot \ldots \cdot \mu(x_n)$, 
so $\mu_n$ is a probability density function on $[-\mfm, \mfm]^n$.
Let $\mbbP_n$ be the probability measure on $[-\mfm, \mfm]^n$ which is induced by $\mu_n$ in the sense that 
if $X \subseteq [-\mfm, \mfm]^n$ is (Lebesque) measurable then 
\[
\mbbP_n(X) = \int_X \mu_n(x_1, \ldots, x_n) dx_1 \ldots dx_n.
\]

\noindent
Call $\mbX \subseteq \mbS_n$ {\em measurable} if the inverse image $\mfS^{-1}(\mbX)$ is a measurable 
subset of $[-\mfm, \mfm]^n$.
By mild abuse of notation we also view $\mbbP_n$ as a probability measure on $\mbS_n$ via the bijection $\mfS_n$
as follows: for measurable $\mbX \subseteq \mbS_n$, 
\[
\mbbP_n(\mbX) := \mbbP_n(\mfS_n^{-1}(\mbX)) := 
\int_{\mfS_n^{-1}(\mbX)} \mu_n(x_1, \ldots, x_n) dx_1 \ldots dx_n.
\]
For all $(x_1, \ldots, x_n) \in [-\mfm, \mfm]^n$
let $Out_n(x_1, \ldots, x_n)$ be the value of the output neuron of $\mcB_n$ in the state $\mfS_n(x_1, \ldots, x_n)$,
or formally, 
\[
Out_n(x_1, \ldots, x_n) := P^{\mfS_n(x_1, \ldots, x_n)}(o_n).
\]

\begin{lem}\label{relevant sets are measurable}
For all $n$ the function $Out_n : [-\mfm, \mfm]^n \to [-\mfm, \mfm]$ is continuous.
It follows that, for every interval $I \subseteq [-\mfm, \mfm]$, $Out_n^{-1}(I)$ is Lebesque measurable.
\end{lem}

\noindent
{\bf Proof.}
Fix any $n \in \mbbN^+$.
For any fixed $w \in [-\mfm, \mfm]$ and $i = 2, \ldots, \rho$, let $g_{i, w}(x) = g_i(x, w)$ so $g_{i, w} : [-\mfm, \mfm] \to [-\mfm, \mfm]$
is continuous.
We prove the claim by induction on $\rho$.
The base case is when $\rho = 2$.
If $w_1, \ldots, w_n$ enumerate the weights of connections from $L_{n, 1} = [n]$ to the unique neuron $o_n$ in $L_{2, 2}$ then,
by~(A9),
\[
P^{\mfS_n(x_1, \ldots, x_n)}(o_n) = F_2(g_{2, w_1}(x_1), \ldots, g_{2, w_n}(x_n))
\]
where the right hand side is a composition of continuous functions in the usual sense of
mathematical analysis, since the aggregation function $F_2$ restricted to $[-\mfm, \mfm]^n$ is continuous in this usual sense because
of part~(1) of Definition~\ref{definition of continuous aggregation function}.
Hence $Out_n$ is continuous if $\rho = 2$.
By the same argument it follows that if $\rho > 2$, $b \in L_{n, 2}$, and
$w_1, \ldots, w_n$ enumerate the weights of connections from $L_{n, 1} = [n]$ to $b$ then 
\[
P^{\mfS_n(x_1, \ldots, x_n)}(b) = F_2(g_{2, w_1}(x_1), \ldots, g_{2, w_n}(x_n))
\]
Then $H_2^b(x_1, \ldots, x_n) := F_2(g_{2, w_1}(x_1), \ldots, g_{2, w_n}(x_n))$ is continuous.
where the right hand side is a continuous function from 
Now suppose that $2 \leq i < \rho$ and that, for all $b \in L_{n, i}$, there is continuous
$H_i^b : [-\mfm, \mfm]^n \to [-\mfm, \mfm]$ such that
\[
P^{\mfS_n(x_1, \ldots, x_n)}(b) = H_i^b(x_1, \ldots, x_n).
\]
Note that we showed that this holds if $i = 2$.
Let $b \in L_{n, i+1}$, let $a_1, \ldots, a_{t(n, i)}$ be an enumeration of $L_{n, i}$,
and let $w_j$ be the weight of the connection from $a_j$ to $b$.
By~(A9) and the assumption we have
\begin{align*}
&P^{\mfS_n(x_1, \ldots, x_n)}(b) = \\
&F_{i+1}\big(g_{i+1, w_1}(H_i^{a_1}(x_1, \ldots, x_n)), \ldots, 
g_{i+1, w_{t(n, i)}}(H_i^{a_{t(n, i)}}(x_1, \ldots, x_n))\big)
\end{align*}
where the right hand side is 
(using part~(1) of Definition~\ref{definition of continuous aggregation function})
a composition of continuous functions, so
the function $H_{i+1}^b(x_1, \ldots, x_n) := P^{\mfS_n(x_1, \ldots, x_n)}(b)$ is continuous.
This completes the inductive step.
\hfill $\square$

\begin{theor}\label{main theorem technical statement}
Suppose that $\mcB_n$, $n \in \mbbN^+$, is a sequence of NNs subject to the assumptions stated above 
((A1) -- (A8))
and that the set of states $\mbS_n$ and the probability distribution $\mbbP_n$ are as defined above.
Then there is $\psi \in [-\mfm, \mfm]$ such that for all $\varepsilon > 0$ there is $c > 0$ such that for all sufficiently large $n$,
\[
\mbbP_n\big(\big\{ \mcA \in \mbS_n : P^\mcA(o_n) \in (\psi - \varepsilon, \psi + \varepsilon) \big\}\big)
\geq 1 - e^{-cn}.
\]
\end{theor}

\noindent
Briefly formulated, the next theorem states that under the above stated assumptions
{\em only two layers suffice} to get,
up to an arbitrarily small error, the same distribution of the output neuron for large $n$.

\begin{cor}\label{two layers suffice}
Suppose $\mcB_n$, $n \in \mbbN^+$, is a sequence of NNs subject to the assumptions stated above and
that $\mbbP_n$ is as stated above.
Let $\psi$ be as in Theorem~\ref{main theorem technical statement}.
Then there is a sequence $\mcB'_n$, $n \in \mbbN^+$, of NNs such that the following hold:
\begin{itemize}
\item Each $\mcB'_n$ has exactly 2 layers and $n$ neurons in the first layer.

\item All assumptions above
hold for $\mcB'_n$ in place of $\mcB_n$ with the same $W_2$ and $g_2$ but with a modified 
continuous aggregation function
$F'_2$ in place of $F_2$ which $\mcB_n$ uses.

\item If $\mbbP'_n$ is the probability distribution induced by $\mu_n$ on the state space $\mbS'_n$ of $\mcB'_n$
and $o'_n$ denotes the output neuron of $\mcB'_n$, then for all $\varepsilon > 0$ there is $c > 0$ such that
for all sufficiently large $n$,
\[
\mbbP'_n\big(\big\{ \mcA \in \mbS'_n : P^\mcA(o'_n) \in (\psi - \varepsilon, \psi + \varepsilon) \big\}\big)
\geq 1 - e^{-cn}.
\]
\end{itemize}
\end{cor}

\begin{exam}\label{necessity of uniform W} {\bf (The necessity of (A7))} 
{\rm
Let $\rho = 2$ and $\mfm = 1$.
Consider a sequence $\mcB_n$, $n \in \mbbN^+$ of NNs satisfying assumptions (A1) -- (A3) and (A6)
(with $\zeta(n) = n$ for example).
Define $g_2 : [-1, 1]^2 \to [-1, 1]$ by $g_2(x, y) = \frac{x + y}{2}$ and
$F_2 : [-1, 1]^{fin} \to [-1, 1]$ by $F_2 = \max$, so (A5) holds.
Let 
\[
\mu(x) = 
\begin{cases}
-16x + 4 \ \ \text{ if } \ 0 \leq x \leq \frac{1}{4},\\
16x + 4 \ \ \text{ if } \ -\frac{1}{4} \leq x  < 0, \\
0 \qquad \text{ otherwise}
\end{cases}
\]
and let
\[
W_2^+(x) = 
\begin{cases}
8x - 4 \ \ \text{ if } \ \frac{1}{2} \leq x \leq 1,\\
0 \qquad \text{ otherwise,}
\end{cases}
\ \ \text{ and } \ \ 
W_2^-(x) = 
\begin{cases}
8x - 4 \ \ \text{ if } \ -1 \leq x \leq -\frac{1}{2},\\
0 \qquad \text{ otherwise.}
\end{cases}
\]
Then (A4) holds if $W_2$ is replaced with any one of $W_2^+$ or $W_2^-$.
Suppose that for every  interval $I \subseteq [-1, 1]$, every $\varepsilon > 0$, and all sufficiently large odd $n$,
\[
\int_I W_2^+(x) dx - \varepsilon \leq 
\frac{\big|\big\{ a \in L_{n, 1} : E^{\mcB_n}(a, o_n) \in I \big\}\big|}{n}
\leq \int_I W_2^+(x) dx + \varepsilon,
\]
and if $\int_I W_2^+(x) dx = 0$, then the proportion above is 0.
Also suppose that for every interval $I \subseteq [-1, 1]$, every $\varepsilon > 0$, and all sufficiently large even $n$,
\[
\int_I W_2^-(x) dx - \varepsilon \leq 
\frac{\big|\big\{ a \in L_{n, 1} : E^{\mcB_n}(a, o_n) \in I \big\}\big|}{n}
\leq \int_I W_2^-(x) dx + \varepsilon,
\]
and if $\int_I W_2^-(x) dx = 0$, then the proportion above is 0.
Let us also assume that the values of neurons are determined as in (A8).
Now all assumptions of Theorem~\ref{main theorem technical statement}, except for (A7), are satisfied.

If $n$ is odd and sufficiently large then, for all $\mcA \in \mbS_n$ and all $a \in L_{n, 1}$, we have 
\[
g_2\big(P^\mcA(a), E^{\mcB_n}(a, o_n)\big) =  \frac{P^\mcA(a) + E^{\mcB_n}(a, o_n)}{2}
\geq \frac{-1/4 + 1/2}{2} = \frac{1}{8}.
\]
If $n$ is even and sufficiently large then,  for all $\mcA \in \mbS_n$ and all $a \in L_{n, 1}$, we have 
\[
g_2\big(P^\mcA(a), E^{\mcB_n}(a, o_n)\big) =  \frac{P^\mcA(a) + E^{\mcB_n}(a, o_n)}{2}
\leq \frac{1/4 - 1/2}{2} = -\frac{1}{8}.
\]
So for all large enough $n$, if $n$ is odd then for all $\mcA \in \mbS_n$, $P^\mcA(o_n) \geq 1/8$,
and if $n$ is even then for all $\mcA \in \mbS_n$, $P^\mcA(o_n) \leq -1/8$.
(We would get the same conclusion if we instead let $F_2$ be the average of the coordinates in the input sequence.)
Thus, the conclusion of Theorem~\ref{main theorem technical statement} does not hold
which shows that we cannot omit Assumption~(A7).
}\end{exam}

\section{Proofs}

\noindent
Let $\mcB_n$, $n \in \mbbN^+$, be a sequence of NNs subject to assumptions~(A1) --~(A8).
We prove Theorem~\ref{main theorem technical statement} by induction on 
the number $\rho$ of layers, where $\rho \geq 2$ (by assumption).

\subsection{The base case: 2 layers}

In this section we prove the following lemma which gives
Theorem~\ref{main theorem technical statement}
in the case when $\rho = 2$:

\begin{lem}\label{the base case}
Suppose that $\rho = 2$, so each $\mcB_n$ has two layers, $L_{n, 1}$ and $L_{n, 2} = \{o_n\}$.
There is $\psi \in [-\mfm, \mfm]$ such that for all $\varepsilon > 0$ there is $c > 0$ such that for all sufficiently large $n$,
\[
\mbbP_n\big(\big\{ \mcA \in \mbS_n : P^\mcA(o_n) \in (\psi - \varepsilon, \psi + \varepsilon) \big\}\big)
\geq 1 - e^{-cn}.
\]
Moreover, $\psi$ depends only on $W_2, F_2$, and $g_2$.
\end{lem}

\noindent
For $\Delta \in \mbbN^+$ and $i \in \{-\Delta, \ldots, \Delta - 1\}$ define the interval
\[
I_{\Delta, i} := \Big[\frac{i \mfm}{\Delta}, \ \frac{(i+1)\mfm}{\Delta}\Big].
\]

\begin{lem}\label{probability of P in an interval}
Let $n, \Delta \in \mbbN^+$, $i \in \{-\Delta, \ldots, \Delta - 1\}$, and $k \in [n]$.
Then 
\begin{align*}
&\mbbP_n\big( \big\{\mcA \in \mbS_n : P^\mcA(k) \in I_{\Delta, i} \big\}\big) = \\
&\mbbP_n\big(\{(x_1, \ldots, x_n) \in [-\mfm, \mfm]^n : x_k \in I_{\Delta, i} \big\}\big) = \int_{I_{\Delta, i}} \mu(x) dx.
\end{align*}
Hence the probability does not depend on $n$ or on $k$, but only on $\Delta$ and $i$ (and $\mu$ which we have fixed).
\end{lem}

\noindent
{\bf Proof.}
From the definitions it directly follows that $x_k \in I_{\Delta, i}$ if and only if $P^{\mfS_n(x_1, \ldots, x_n)}(k) \in I_{\Delta, i}$.
Since for all $X \subseteq [-\mfm, \mfm]^n$, $\mbbP_n(X) = \mbbP_n(\mfS_n(X))$ it suffices to prove the second
equality of the lemma.

By the definition of $\mbbP_n$ and since $\mu_n(x_1, \ldots, x_n) = \mu(x_1)\cdots  \ldots \cdot \mu(x_n)$
where $\mu : [-\mfm, \mfm] \to [0, \infty)$ is a probability density function, we get, by using Fubini's theorem,
\begin{align*}
&\mbbP_n\big(\{(x_1, \ldots, x_n) \in [-\mfm, \mfm]^n : x_k \in I_{\Delta, i} \big\}\big) = \\
&\int_{[-\mfm, \mfm]^{k-1} \times I_{\Delta, i} \times [-\mfm, \mfm]^{n-k}} \mu_n(x_1, \ldots, x_n) dx_1 \ldots dx_n = \\
&\int_{[-\mfm, \mfm]^{k-1} \times I_{\Delta, i} \times [-\mfm, \mfm]^{n-k}} \mu(x_1) \cdots  \ldots \cdot \mu(x_n) dx_1 \ldots dx_n \\
&\bigg(\prod_{i = 1}^{k-1}\int_{[-\mfm, \mfm]} \mu(x_i) dx_i \bigg) \cdot
\int_{I_{\Delta, i}} \mu(x_k) dx_k \cdot 
\bigg(\prod_{i = k+1}^{n}\int_{[-\mfm, \mfm]} \mu(x_i) dx_i \bigg) = \\
&\int_{I_{\Delta, i}} \mu(x_k) dx_k.
\end{align*}
\hfill $\square$

\begin{lem}\label{independence of coordinates}
Let $n \in \mbbN^+$, $k \in [n]$, and let $I_k \subseteq [-\mfm, \mfm]$ be an interval.
The event $\{\mcA \in \mbS_n : P^\mcA(k) \in I_k\}$ is independent from from all events
$\{\mcA \in \mbS_n : P^\mcA(l) \in I_i\}$ where $i \neq k$ and $I_i \subseteq [-\mfm, \mfm]$ is an interval.
\end{lem}

\noindent
{\bf Proof.}
By the definitions, 
\begin{align*}
&\mbbP_n\big(\big\{\mcA \in \mbS_n : P^\mcA(i) \in I_i \text{ for all } i \in [n]\big\}\big) = \\
&\mbbP_n\big(\{(x_1, \ldots, x_n) \in [-\mfm, \mfm]^n : x_i \in I_i \text{ for all } i \in [n]\big\}\big).
\end{align*}
Therefore it suffices to prove that 
the event $\{(x_1, \ldots, x_n) \in [-\mfm, \mfm]^n : x_k \in I_k\}$ is independent from from all events
$\{(x_1, \ldots, x_n) \in [-\mfm, \mfm]^n : x_i \in I_i\}$ where $i \neq k$ and $I_i \subseteq [-\mfm, \mfm]$ is an interval.
But this follows by using Fubini's therem since $\mu_n(x_1, \ldots, x_n) = \mu(x_1) \cdot \ldots \cdot \mu(x_k)$.
\hfill $\square$

\medskip

\noindent
For $\Delta \in \mbbN^+$ and $i \in \{-\Delta, \ldots, \Delta - 1\}$ let 
\[
\beta_{\Delta, i} := \int_{I_{\Delta, i}} \mu(x) dx \qquad \text{ and } \qquad
\gamma_{\Delta, i} := \int_{I_{\Delta, i}} W_2(x) dx.
\]
Also define
\begin{equation}\label{def of B-n-delta-i}
B_n(\Delta, i) := \big\{k \in [n] : E^{\mcB_n}(k, o_n) \in I_{\Delta, i} \big\}.
\end{equation}
It follows from~(A7) that for all $\varepsilon > 0$,
if $n$ is large enough then
\begin{equation}\label{B-i over L-1}
\gamma_{\Delta, i} - \varepsilon \leq \frac{|B_n(\Delta, i)|}{n} \leq \gamma_{\Delta, i} + \varepsilon.
\end{equation}
So if $\gamma_{\Delta, i} > 0$, then 
$|B_n(\Delta, i)| \to \infty$ as $n \to \infty$ and the rate och growth is linear in $n$.
Also, by~(A7), if $\gamma_{\Delta, i} = 0$, then $B_n(\Delta, i) = \es$ .

For all $\mcA \in \mbS_n$, $\Delta \in \mbbN^+$ and $i, j \in \{-\Delta, \ldots, \Delta - 1\}$, let 
\begin{equation}\label{def of C-n-delta-i-j}
C_n(\mcA, \Delta, i, j) := \big\{k \in B_n(\Delta, i) : P^\mcA(k) \in I_{\Delta, j} \big\}.
\end{equation}
For $\Delta \in \mbbN^+$ and $\varepsilon > 0$ define
\begin{align}\label{def of X-n}
&\mbX_n(\Delta, \varepsilon) := \\
&\big\{ \mcA \in \mbS_n : \text{ for all $i, j \in \{-\Delta, \ldots, \Delta - 1\}$ such that $\gamma_{\Delta, i} > 0$}, 
\nonumber \\
&(\beta_{\Delta, j} - \varepsilon)|B_n(\Delta, i)| \leq 
\big| C_n(\mcA, \Delta, i, j) \big|
\leq (\beta_{\Delta, j} + \varepsilon)|B_n(\Delta, i)|\big\}.
\nonumber
\end{align}

\begin{lem}\label{limit of X-n}
For all $n, \Delta \in \mbbN^+$ and $\varepsilon > 0$, $\mbX_n(\Delta, \varepsilon)$ is measurable.
If $n$ is sufficiently large, then
$\mbbP_n\big(\mbX_n(\Delta, \varepsilon)\big) \geq 1 - e^{-c n}$ 
for some constant $c > 0$ that depends only on $\varepsilon$.
\end{lem}

\noindent
{\bf Proof.}
Fix any $i, j \in \{-\Delta, \ldots, \Delta-1\}$.
Define
\begin{align*}
&\mbX_n(\Delta, \varepsilon, i, j) := \\
&\big\{ \mcA \in \mbS_n : 
(\beta_{\Delta, j} - \varepsilon)|B_n(\Delta, i)| \leq 
\big| C_n(\mcA, \Delta, i, j) \big|
\leq (\beta_{\Delta, j} + \varepsilon)|B_n(\Delta, i)|
 \big\}.
\end{align*}
From the definition of measurable subset of $\mbS_n$ it follows that every subset of $\mbS_n$ of the form
\[
\mbX_n(k) := \big\{\mcA \in \mbS_n : P^\mcA(k) \in I_{\Delta, i} \big\} \ \text{ for some } k \in [n]
\]
is measurable.
Since $\mbX_n(\Delta, \varepsilon, i, j)$ can be formed from sets of the form $\mbX_n(k)$ by
taking complements, intersections, and unions, a finite number of times,
it follows that $\mbX_n(\Delta, \varepsilon, i, j)$ is measurable.
As $\mbX_n(\Delta, \varepsilon)$ is a finite union of sets of the form $\mbX_n(\Delta, \varepsilon, i, j)$
it follows that also $\mbX_n(\Delta, \varepsilon)$ is measurable.

From Lemma~\ref{probability of P in an interval}
it follows that, for all $k \in [n]$,
\begin{equation}\label{probability of delta i j}
\mbbP_n\big(\big\{ \mcA \in \mbS_n : P^\mcA(k) \in I_{\Delta, j} \big\}\big) = \beta_{\Delta, j}
 := \int_{I_{\Delta, j}} \mu(x) dx.
\end{equation}

First suppose that $\gamma_{\Delta, i} > 0$.
Then, as noted above, $\lim_{n \to \infty} |B_n(\Delta, i)| = \infty$ so it follows
from \eqref{probability of delta i j}, Lemma~\ref{independence of coordinates}, 
and Lemma~\ref{independent bernoulli trials}  that there is $d > 0$ such that if $n$ is large enough, then
\begin{equation}\label{X-n-delta-i-j almost 1}
\mbbP_n\big(\mbX_n(\Delta, \varepsilon, i, j) \big) \geq 1 - e^{-dn}.
\end{equation}
Now suppose that  $\gamma_{\Delta, i} = 0$.
Then, as noted above, $B_n(\Delta, i) = \es$ and as, for all $\mcA \in \mbS_n$,  
$C_n(\mcA, \Delta, i, j) \subseteq B_n(\Delta, i)$ we get $C_n(\mcA, \Delta, i, j) = \es$.
Hence $\mbX_n(\Delta, \varepsilon, i, j) = \mbS_n$ so~\eqref{X-n-delta-i-j almost 1} holds for any choice of $d > 0$.
Since 
\[
\mbX_n(\Delta, \varepsilon) = 
\bigcup_{-\Delta \leq i, j \leq \Delta - 1} 
\mbX_n(\Delta, \varepsilon, i, j)
\]
and there are only $4\Delta^2$ choices of $i$ and $j$ it follows that there is $c > 0$ such that if $n$ is sufficiently large, then
\[
\mbbP_n\big(\mbX_n(\Delta, \varepsilon) \big) \geq 1 - e^{-cn}.
\]
\hfill $\square$

\medskip

\noindent
Note that 
\[
L_{n, 1} = \bigcup_{i=-\Delta}^{\Delta - 1} B_n(\Delta, i)
\]
and observe that for every $\mcA \in \mbS_n$ we have
\[
B_n(\Delta, i) = \bigcup_{j=-\Delta}^{\Delta - 1} C_n(\mcA, \Delta, i, j) \quad \text{and} \quad
L_{n, 1} = \bigcup_{i=-\Delta}^{\Delta - 1} \ \bigcup_{j=-\Delta}^{\Delta - 1} C_n(\mcA, \Delta, i, j).
\]
For $\mcA \in \mbX_n(\Delta, \delta)$ and $i, j \in \{-\Delta, \ldots, \Delta-1\}$ such that $\gamma_{\Delta, i} > 0$ we have
\begin{equation}\label{C-i-j over B-i}
\beta_{\Delta, j} - \delta \leq  \frac{|C_n(\mcA, \Delta, i, j)|}{|B_n(\Delta, i)|} \leq \beta_{\Delta, j} + \delta.
\end{equation}
Note that if $\beta_{\Delta, j}, \gamma_{\Delta, i} > 0$ then
\[
\frac{|C_n(\mcA, \Delta, i, j)|}{n}  =
\frac{|B_n(\Delta, i)|}{n} \cdot \frac{|C_n(\mcA, \Delta, i, j)|}{|B_n(\Delta, i)|}
\leq \gamma_{\Delta, i} \beta_{\Delta, j} + 3\delta.
\]
In a similar way we get $\frac{|C_n(\mcA, \Delta, i, j)|}{n} \geq \gamma_{\Delta, i} \beta_{\Delta, j} - 3\delta$.
Thus
\begin{equation}\label{bounds of C-n over n in positive case}
\beta_{\Delta, j}, \gamma_{\Delta, i} > 0  \ \Longrightarrow \ 
\gamma_{\Delta, i} \beta_{\Delta, j} - 3\delta \leq
\frac{|C_n(\mcA, \Delta, i, j)|}{n} \leq 
\gamma_{\Delta, i} \beta_{\Delta, j} + 3\delta.
\end{equation}
On the other hand, 
\begin{equation}\label{bounds of C-n over n in zero case}
\text{if  $\beta_{\Delta, j} = 0$ or $\gamma_{\Delta, i} = 0$  \ \ then } \ \ 
\frac{|C_n(\mcA, \Delta, i, j)|}{n} = 0.
\end{equation}

\noindent
Fix $M \in \mbbN^+$ and $l \in \{-M, \ldots, M-1\}$.
Let
\begin{align*}
&\Omega_0 := \big\{ (x, y) \in [-\mfm, \mfm]^2 :  \mu(x) > 0 \text{ and } W_2(y) > 0 \big\}, \ \text{ and}\\
&\text{let $\Omega$ be the closure of $\Omega_0$.}
\end{align*}
Then $\Omega$ is a finite union of closed rectangles.
Now let 
\[
X_l := g_2^{-1}(I_{M, l}) \cap \Omega.
\]
Note that $X_l$ does not depend on $\Delta$ or $n$.
By $\lambda^2$ we denote the Lebesgue measure on $\mbbR^2$ restricted to $[-\mfm, \mfm]^2$.
Let 
\[
\alpha_l := \int_{X_l} \mu(x) W_2(y) dx dy.
\]
Observe that for all $\Delta \in \mbbN^+$ and $i, j \in \{-\Delta, \ldots, \Delta - 1\}$,
\begin{equation}\label{integral for gamma times beta}
\gamma_{\Delta, i}\beta_{\Delta, j} = 
\int_{I_{\Delta, j}} \mu(x) dx \
\int_{I_{\Delta, i}} W_2(y) dy 
=
\int_{I_{\Delta, j} \times I_{\Delta, i}} \mu(x) W_2(y) dx dy.
\end{equation}
For $\mcA \in \mbS_n$  and $k \in [n]$ let
\begin{equation}\label{def of q-A}
q_k(\mcA) := g_2(P^\mcA(k), E^{\mcB_n}(k, o_n))
\end{equation}
and
\begin{equation}\label{def of omega-M-l}
\omega(\mcA, M, l) := \frac{\big|\big\{k \in [n] : q_k(\mcA) \in I_{M, l} \big\}\big|}{n}.
\end{equation}
The next step is to prove the following:

\begin{claim}\label{verification of 2b}
For every $\delta' > 0$, if $\Delta \in \mbbN^+$ is large enough and $\delta > 0$ small enough 
and $\mcA \in \mbX_n(\Delta, \delta)$, then
\begin{equation}\label{verification of 2b}
\alpha_l - \delta' \leq \omega(\mcA, M, l) \leq \alpha_l + \delta'.
\end{equation}
\end{claim}

\noindent
{\bf Proof.}
Define
\begin{align*}
\Gamma_1 := \big\{ (i, j) \in \{-\Delta, \ldots, \Delta-1\}^2 : \ &\beta_{\Delta, j}, \gamma_{\Delta, i} > 0
\text{ and } g_2(I_{\Delta, j} \times I_{\Delta, i}) \subseteq I_{M, l} \big\} \ \ \text{ and} \\
\Gamma_2 := \big\{ (i, j) \in \{-\Delta, \ldots, \Delta-1\}^2 : \ 
&\beta_{\Delta, j}, \gamma_{\Delta, i} > 0 \text{ and }\\
&g_2(I_{\Delta, j} \times I_{\Delta, i}) \cap I_{M, l} \neq \es
\text{ and } g_2(I_{\Delta, j} \times I_{\Delta, i}) \not\subseteq I_{M, l} \big\}.
\end{align*}
Then let
\begin{align*}
Z_1 := \bigcup_{(i, j) \in \Gamma_1}  I_{\Delta, j} \times I_{\Delta, i} \quad \text{ and} \quad
Z_2 := \bigcup_{(i, j) \in \Gamma_2}  I_{\Delta, j} \times I_{\Delta, i}.
\end{align*}
By definition of $\Gamma_1$ we have $Z_1 \subseteq X_l$,
$\lambda^2(Z_1) \leq \lambda^2(X_l)$, and $\lim_{\Delta \to \infty} \lambda^2(Z_1) = \lambda^2(X_l)$.

Let $Y_l$ be the boundary of $X_l$ relative to $[-\mfm, \mfm]^2$.
That is, $Y_l$ contains all $(x, y) \in [-\mfm, \mfm]^2$ such that for every $\delta' > 0$, the ball with radius $\delta'$
centered at $(x, y)$ contains a point in $X_l$ and a point not in $X_l$.
Then $Y_l$ is a finite union of curves, so $\lambda^2(Y_l) = 0$.
Note that $(i, j) \in \Gamma_2$ if and only if $I_{\Delta, j} \times I_{\Delta, i} \cap Y_l \neq \es$.
Hence $\lim_{\Delta \to \infty} \lambda^2(Z_2) = 0$.

Suppose that $\mcA \in \mbX_n(\Delta, \delta)$.
Since $\lim_{\Delta \to \infty} \lambda^2(Z_1) = \lambda^2(X_l)$ it follows that,
for every $\delta' > 0$, if $\Delta$ is large enough then we have
\begin{align}\label{lower bound of omega}
&\omega(\mcA, M, l) 
\geq \sum_{(i, j) \in \Gamma_1} \frac{|C_n(\mcA, \Delta, i, j)|}{n} 
\geq \sum_{(i, j) \in \Gamma_1} (\gamma_{\Delta, i} \beta_{\Delta, j} - 3\delta) 
\quad \text{(by \eqref{bounds of C-n over n in positive case})}\\
&\geq \sum_{(i, j) \in \Gamma_1} \gamma_{\Delta, i} \beta_{\Delta, j} 
= \sum_{(i, j) \in \Gamma_1} \int_{I_{\Delta, j} \times I_{\Delta, i}} \mu(x) W_2(y) dx dy
\quad \text{(by \eqref{integral for gamma times beta})}
\nonumber\\
&= \int_{Z_1} \mu(x) W_2(y) dx dy
\geq \int_{X_l} \mu(x) W_2(y) dx dy - \delta'
= \alpha_l - \delta'.
\nonumber
\end{align}
Note that if $g_2(I_{\Delta, j} \times I_{\Delta_i}) \cap I_{M, l} = \es$ then there is {\em no}
$k \in C_n(\mcA, \Delta, i, j)$ such that $q_k(\mcA) \in I_{M, l}$.
By also using~\eqref{bounds of C-n over n in zero case} we get
\begin{align*}
&\omega(\mcA, M, l) 
\leq \sum_{(i, j) \in \Gamma_1} \frac{|C_n(\mcA, \Delta, i, j)|}{n} 
+ \sum_{(i, j) \in \Gamma_2} \frac{|C_n(\mcA, \Delta, i, j)|}{n} \\
&\leq \sum_{(i, j) \in \Gamma_1} (\gamma_{\Delta, i} \beta_{\Delta, j} + 3\delta) 
+ \sum_{(i, j) \in \Gamma_2} (\gamma_{\Delta, i} \beta_{\Delta, j} + 3\delta) \nonumber \\
&= \sum_{(i, j) \in \Gamma_1} \gamma_{\Delta, i} \beta_{\Delta, j} 
+ \sum_{(i, j) \in \Gamma_2} \gamma_{\Delta, i} \beta_{\Delta, j} + 3\delta|\Gamma_1| + 3\delta|\Gamma_2| 
\nonumber \\
&\leq \sum_{(i, j) \in \Gamma_1} \gamma_{\Delta, i} \beta_{\Delta, j} +
\sum_{(i, j) \in \Gamma_2} \gamma_{\Delta, i} \beta_{\Delta, j} + 
24\delta \Delta^2 
\nonumber\\
&= \sum_{(i, j) \in \Gamma_1} \int_{I_{\Delta, j} \times I_{\Delta, i}} \mu(x) W_2(y) dx dy 
\nonumber \\
&+
\sum_{(i, j) \in \Gamma_2} \int_{I_{\Delta, j} \times I_{\Delta, i}} \mu(x) W_2(y) dx dy + 
24\delta \Delta^2 \quad \text{(by \eqref{integral for gamma times beta})} 
\nonumber \\
&= \int_{Z_1} \mu(x) W_2(y) dx dy + \int_{Z_2} \mu(x) W_2(y) dx dy + 24\delta \Delta^2
\nonumber \\
&\leq \int_{X_l} \mu(x) W_2(y) dx dy + \int_{Z_2} \mu(x) W_2(y) dx dy + 24\delta \Delta^2 
\quad \text{(as $Z_1 \subseteq X_l$)}.
\nonumber
\end{align*}
Since $\lim_{\Delta \to \infty} \lambda^2(Z_2) = 0$ it follows from the above that, for every $\delta' > 0$,
if $\Delta$ is sufficiently large and $\delta > 0$ sufficiently small then
\begin{equation}\label{upper bound of omega}
\omega(\mcA, M, l) \leq \int_{X_l} \mu(x) W_2(y) dx dy + \delta' = \alpha_l + \delta'.
\end{equation}
By combining~\eqref{lower bound of omega}
and~\eqref{upper bound of omega}
we get Claim~\ref{verification of 2b}. 
\hfill $\square$
\medskip

\noindent
The next step is to prove:

\begin{claim}\label{verification of 2c and d}
Let $\mcA \in \mbW_n$.\\
(a) If $-M \leq l < M-2$ and $\alpha_l = \alpha_{l+1} = \alpha_{l+2} = 0$, 
then there is no $k \in [n]$ such that $q_k(\mcA) \in I_{M, l+1}$.\\
(b) If $l = M-2$ and $\alpha_l = \alpha_{l+1} = 0$, 
then there is no $k \in [n]$ such that $q_k(\mcA) \in I_{M, l+1}$.\\
(c) If $l = -M$ and $\alpha_l = \alpha_{l+1} = 0$,
then there is no $k \in [n]$ such that $q_k(\mcA) \in I_{M, l}$.
\end{claim}

\noindent
{\bf Proof.}
(a) Suppose that $\alpha_l = \alpha_{l+1} = \alpha_{l+2} = 0$.
Suppose for a contradiction that $g_2(\Omega) \cap I_{M, l+1} \neq \es$.
Then there is $z \in \Omega$ such that $g_2(z) \in I_{M, l+1}$.
Recall that $\Omega$ is a finite union of closed rectangles.
Let $\Omega_z$ denote the connected component to which $z$ belongs.
Suppose, towards a contradiction, that there is $z' \in \Omega_z$ such that $g_2(z') \in I_{M, l+1}$ and $g_2(z') \neq g_2(z)$.
By the continuity of $g_2$, $g_2^{-1}(I_{M, l+1}) \cap \Omega_z$ has positive measure and
hence also $X_{l+1}$ has positive measure.
Since $\mu(x) W_2(y)$ is positive on the interior of $\Omega$, hence on the interior of $X_{l+1}$,
it follows that $\alpha_{l+1} = \int_{X_{l+1}} \mu(x) W_2(y) dx dy > 0$ which contradicts the assumption that 
$\alpha_{l+1} = 0$.
Thus we conclude that $g_2(\Omega_z) \cap I_{M, l+1}$ is a singleton.
There is some $k$ such that $g_2^{-1}(I_{M, k}) \cap \Omega_z$ has positive measure, and hence $\alpha_k > 0$
and $g_2(\Omega_z) \cap I_{M, k} \neq \es$.
Suppose that $k < l$. (The case $l+2 < k$ is treated analogously.)
Since we assume that $\alpha_l = 0$ we can argue as we did for $\alpha_{l+1}$ and conclude that
$g_2(\Omega_z) \cap I_{M, l}$ is either empty or a singleton.
Since  $g_2(\Omega_z) \cap I_{M, k} \neq \es$, $g_2(\Omega_z) \cap I_{M, l+1} \neq \es$, and $g_2$ is continuous,
it follows that $g_2 \uhrc \Omega_z$ must assume all values in $I_{M, l}$.
This contradicts the previous conclusion that $g_2(\Omega_z) \cap I_{M, l}$ is either empty or a singleton.
Hence we conclude that $g_2(\Omega) \cap I_{M, l+1} = \es$.
It follows that $q_k(\mcA) \notin I_{M, l+1}$ for all $k \in [n]$.

(b) Suppose that $l = M-2$ and $\alpha_l = \alpha_{l+1} = 0$. 
As the argument is similar to that in part (a) it is only sketched.
Towards a contradiction, suppose that  there is $z \in \Omega$ such that $g_2(z) \in I_{M, l+1}$.
Let $\Omega_z$ be the connected component of $\Omega$ to which $z$ belongs.
By an argument as in case~(a), $g_2(\Omega_z) \cap I_{M, l+1}$ is a singleton.
There is some $k$ such that $g_2^{-1}(I_{M, k}) \cap \Omega_z$ has positive measure, and hence $\alpha_k > 0$
and $g_2(\Omega_z) \cap I_{M, k} \neq \es$.
Hence $k < l$. Now one can argue just as in case~(a) to derive a contradiction.
Thus we conclude that $g_2(\Omega) \cap I_{M, l+1} = \es$.

Case (c) is proved similarly as case (b) so we omit it.
\hfill $\square$

\medskip

\noindent
Since the aggregation function $F_2$ is continuous in the sense of 
Definition~\ref{definition of continuous aggregation function}
it follows from Claims~\ref{verification of 2b} and~\ref{verification of 2c and d}
(which are valid for all $M \in \mbbN^+$ and $l \in \{-\Delta, \ldots, \Delta-1\}$)
that, for every $\varepsilon > 0$, if $\Delta$ is sufficiently large, $\delta > 0$ is sufficiently small, 
$n$ and $m$ are sufficiently large, $\mcA \in \mbX_n(\Delta, \delta)$, $\mcA' \in \mbX_m(\Delta, \delta)$,
$\bar{q} = (q_1(\mcA), \ldots, q_n(\mcA))$, and $\bar{q}' = (q_1(\mcA'), \ldots, q_m(\mcA'))$,
then $|F_2(\bar{q}) - F_2(\bar{q}')| < \varepsilon$.
Since $P^\mcA(o_n) = F_2(\bar{q})$ and $P^{\mcA'}(o_m) = F_2(\bar{q}')$
we get $|P^\mcA(o_n) - P^{\mcA'}(o_m)| \leq \varepsilon$.
As $\mbbP_n\big(\mbX_n(\Delta, \delta)\big) \geq 1 - e^{-cn}$
(where $c > 0$) for all choices of $\Delta$ and $\delta$ and all sufficiently large $n$
it follows that there is $\psi \in [-\mfm, \mfm]$ such that,
for every open interval $I$ containing $\psi$, 
\[
\mbbP_n\big(\big\{ \mcA \in \mbS_n : P^\mcA(o_n) \in I\big\}\big) \geq 1 - e^{-cn}.
\]
This concludes the proof of 
Lemma~\ref{the base case}.

\subsection{The induction step}

Suppose that $\rho \geq 3$.

\medskip
\noindent
{\bf Induction hypothesis.}
{\em If $\mcB_n$, $n \in \mbbN^+$, is a sequence of NNs with $\rho - 1$ layers,
then there is $\psi \in [-\mfm, \mfm]$ such that for every $\varepsilon > 0$ there is $c > 0$ such that for
all sufficiently large $n$,
\[
\mbbP_n\big(\big\{ \mcA \in \mbS_n : P^\mcA(o_n) \in (\psi - \varepsilon, \psi + \varepsilon) \big\}\big)
\geq 1 - e^{-cn}.
\]
Moreover, $\psi$ depends only on $W_i, F_i$ and $g_i$ for $i = 2, \ldots, \rho - 1$.}

\medskip

\noindent
Suppose that $\mcB_n$, $n \in \mbbN^+$, s a sequence of NNs with $\rho$ layers.

\begin{lem}\label{stronger induction hypothesis}
There is $\psi \in [-\mfm, \mfm]$ such that for every $\varepsilon > 0$ there is $c > 0$ such that
for all sufficently large $n$,
\[
\mbbP_n\big(\big\{ \mcA \in \mbS_n : 
\text{ for all } a \in L_{n, \rho-1}, P^\mcA(a) \in (\psi - \varepsilon, \psi + \varepsilon) \big\}\big)
\geq 1 - e^{-cn}.
\]
Moreover, $\psi$ depends only on $W_i, F_i$ and $g_i$ for $i = 2, \ldots, \rho - 1$.
\end{lem}

\noindent
{\bf Proof.}
For every $n$ and an arbitrary choice of $a_n \in L_{n, \rho - 1}$ we construct an NN $\mcB'_n$ with
only $\rho - 1$ layers as follows.
Remove layer $\rho$ (i.e. $L_{n, \rho}$) from $\mcB_n$ and remove all neurons in $L_{n, \rho - 1}$
except for $a_n$. So the layer $\rho - 1$ of $\mcB'_n$ is $L'_{n, \rho - 1} :=  \{a_n\}$.
All other layers are left unchanged and $\mcB'_n$ will keep $W_i, F_i$ and $g_i$ for $i = 2, \ldots, \rho - 1$.
So each $\mcB'_n$ is an NN with $\rho - 1$ layers and a unique output neuron.
Let $\mbS'_n$ be the set of states of $\mcB'_n$ and $\mbbP'_n$ the probability measure 
on $\mbS'_n$ defined  just as $\mbbP_n$ was defined on $\mbS_n$.
By the induction hypothesis there is 
$\psi \in [-\mfm, \mfm]$ such that for every $\varepsilon > 0$ there is $c > 0$ such that for
all sufficiently large $n$,
\[
\mbbP'_n\big(\big\{ \mcA \in \mbS'_n : P^\mcA(a_n) \in (\psi - \varepsilon, \psi + \varepsilon) \big\}\big)
\geq 1 - e^{-cn}
\]
and $\psi$ depends only on $W_i, F_i$ and $g_i$ for $i = 2, \ldots, \rho - 1$.
Because of the definitions of $\mcB'_n$, $\mbS'_n$ and $\mbbP'_n$, we have
\begin{align*}
&\mbbP'_n\big(\big\{ \mcA \in \mbS'_n : P^\mcA(a_n) \in (\psi - \varepsilon, \psi + \varepsilon) \big\}\big) = \\
&\mbbP_n\big(\big\{ \mcA \in \mbS_n : P^\mcA(a_n) \in (\psi - \varepsilon, \psi + \varepsilon) \big\}\big),
\end{align*}
so
\[
\mbbP_n\big(\big\{ \mcA \in \mbS_n : P^\mcA(a_n) \in (\psi - \varepsilon, \psi + \varepsilon) \big\}\big)
\geq 1 - e^{-cn}.
\]
Since $a_n$ can be any neuron in $L_{n, \rho - 1}$ and $|L_{n, \rho - 1}| \leq \zeta(n)$,
by Assumption~(A6),
it follows that 
\[
\mbbP_n\big(\big\{ \mcA \in \mbS_n : 
\exists a \in L_{n, \rho - 1}, 
P^\mcA(a) \notin (\psi - \varepsilon, \psi + \varepsilon) \big\}\big) \ \leq \ \zeta(n) \cdot e^{-cn}.
\]
Since $\zeta$ is a polynomial there is $c' > 0$ such that for all sufficiently large $n$,
$\zeta(n) \cdot e^{-cn} \leq e^{-c' n}$, and therefore
\[
\mbbP_n\big(\big\{ \mcA \in \mbS_n : 
\forall a \in L_{n, \rho - 1}, P^\mcA(a) \in (\psi - \varepsilon, \psi + \varepsilon) \big\}\big)
\geq 1 - e^{-c'n}.
\]
This completes the proof of the lemma.
\hfill $\square$

\bigskip

\noindent
We will prove that there is $\varphi \in [-\mfm, \mfm]$ such that for every $\varepsilon > 0$ there is $c > 0$ such that for
all sufficiently large $n$,
\begin{equation*}\label{concentration around phi}
\mbbP_n\big(\big\{ \mcA \in \mbS_n : P^\mcA(o_n) \in (\varphi - \varepsilon, \varphi + \varepsilon) \big\}\big)
\geq 1 - e^{-cn}.
\end{equation*}
Moreover, $\varphi$ will only depend on $W_i, F_i$ and $g_i$ for $i = 2, \ldots, \rho$.
Let 
\[
\text{$L_{n, \rho - 1} = \{a_1, \ldots, a_{t(n)}\}$, so $|L_{n, \rho - 1}| = t(n)$.}
\]
Let us use the following abbreviations:
\[
\text{ $W := W_\rho, \ F := F_\rho$, and $g := g_\rho$.}
\]

\noindent
For any $n$, $\mcA \in \mbS_n$ and $k \in [t(n)]$, let
\begin{equation*}
q_k(\mcA) := g(P^\mcA(a_k), E^{\mcB_n}(a_k, o_n))
\end{equation*}
and, for $M \in \mbbN^+$ and $l \in \{0, \ldots, M-1\}$,
\begin{equation*}
\omega(\mcA, M, l) := \frac{\big|\big\{k \in [t(n)] : q_k(\mcA) \in I_{M, l} \big\}\big|}{t(n)}.
\end{equation*}

Let $\psi$ be as in Lemma~\ref{stronger induction hypothesis},
so $\psi$ depends only on $W_i, F_i$ and $g_i$ for $i = 2, \ldots, \rho - 1$.
For $n \in \mbbN^+$ and $\delta > 0$ define
\[
\mbX_n(\delta) := \big\{ \mcA \in \mbS_n : \text{ for all } a \in L_{n, \rho - 1}, \
P^\mcA(a) \in (\psi - \delta, \psi + \delta) \big\}.
\]
Let 
\begin{align*}
\Omega_0 := \big\{ x \in [-\mfm, \mfm] : W(x) > 0 \big\} \ \text{ and } \ 
\text{let $\Omega$ be the closure of $\Omega_0$.}
\end{align*}
Note that $\Omega$ is a finite union of closed intervals.
Fix any $M \in \mbbN^+$ and for all $l \in \{-M, \ldots, M-1\}$ define
\[
X_l := g_\psi^{-1}(I_{M, l}) \cap \Omega \quad \text{ and } \quad \alpha_l := \int_{X_l} W(x) dx.
\]
Note that $X_l$ is a finite union of closed intervals.
Define $g_\psi : [-\mfm, \mfm] \to [-\mfm, \mfm]$ by $g_\psi(x) := g(\psi, x)$ (so $g_\psi$ does not depend on $M$ or $l$).

By Assumption~(A7), for all $\delta' > 0$, the proportion of $a_k \in L_{n, \rho}$ such that $E^{\mcB_n}(a_k, o_n) \in X_l$
is in $(\alpha_l - \delta', \alpha_l + \delta')$ if $n$ is large enough.
Since $g$ is continuous it follows that (for every choice of $M \in \mbbN^+$)
\begin{align}\label{verification of 2b in inductive step}
&\text{for every $\delta > 0$, if $\delta' > 0$ is sufficiently small, $n$ sufficiently large,} \\
&\text{and $\mcA \in \mbX_n(\delta')$, then $\alpha_l - \delta \leq \omega(\mcA, M, l) \leq \alpha_l + \delta$ for all $l = -M, \ldots, M-1$.}
\nonumber
\end{align}
We now prove:

\begin{claim}\label{verification of 2c and d in induction step}
Let $\mcA \in \mbW_n$.\\
(a) If $-M \leq l < M-2$ and $\alpha_l = \alpha_{l+1} = \alpha_{l+2} = 0$, 
then there is no $k \in [n]$ such that $q_k(\mcA) \in I_{M, l+1}$.\\
(b) If $l = M-2$ and $\alpha_l = \alpha_{l+1} = 0$, 
then there is no $k \in [n]$ such that $q_k(\mcA) \in I_{M, l+1}$.\\
(c) If $l = -M$ and $\alpha_l = \alpha_{l+1} = 0$,
then there is no $k \in [n]$ such that $q_k(\mcA) \in I_{M, l}$.
\end{claim}

\noindent
{\bf Proof.}
(a) The proof is parallell to that of part (a) of Claim~\ref{verification of 2c and d}.
Suppose that $-M \leq l < M-2$ and $\alpha_l = \alpha_{l+1} = \alpha_{l+2} = 0$.
Suppose for a contradiction that there is $z \in \Omega$ such that $g(z) \in I_{M, l+1}$.
Recall that $\Omega$ is a finite union of closed intervals.
Let $\Omega_z$ denote the connected component to which $z$ belongs.
Suppose, towards a contradiction, that there is $z' \in \Omega_z$ such that $g(z') \in I_{M, l+1}$ and $g(z') \neq g_2(z)$.
By the continuity of $g$, $g^{-1}(I_{M, l+1}) \cap \Omega_z$ has positive measure (as a subset of $\mbbR$) and
hence also $X_{l+1}$ has positive measure.
Since $W_2(x)$ is positive on the interior of $\Omega$, hence on the interior of $X_{l+1}$,
it follows that $\alpha_{l+1} = \int_{X_{l+1}} W_2(x) dx > 0$ which contradicts the assumption that 
$\alpha_{l+1} = 0$.
Thus we conclude that $g(\Omega_z) \cap I_{M, l+1}$ is a singleton.
There is some $k$ such that $g^{-1}(I_{M, k}) \cap \Omega_z$ has positive measure
(as a subset of $\mbbR$), and hence $\alpha_k > 0$
and $g(\Omega_z) \cap I_{M, k} \neq \es$.
Suppose that $k < l$. (The case $l+2 < k$ is treated analogously.)
Since we assume that $\alpha_l = 0$ we can argue as we did for $\alpha_{l+1}$ and conclude that
$g(\Omega_z) \cap I_{M, l}$ is either empty or a singleton.
Since  $g(\Omega_z) \cap I_{M, k} \neq \es$, $g(\Omega_z) \cap I_{M, l+1} \neq \es$, and $g$ is continuous,
it follows that $g \uhrc \Omega_z$ must assume all values in $I_{M, l}$.
This contradicts the previous conclusion that $g(\Omega_z) \cap I_{M, l}$ is either empty or a singleton.
Hence we conclude that $g(\Omega) \cap I_{M, l+1} = \es$.
It follows that $q_k(\mcA) \notin I_{M, l+1}$ for all $k \in [n]$.

Parts (b) and (c) are proved in a similar way 
(see the proof of part~(b) of Claim~\ref{verification of 2c and d}) so their proofs are omitted.
\hfill $\square$

\medskip

\noindent
Let $\varepsilon > 0$.
As the aggregation function $F$ is continuous it follows 
from~\eqref{verification of 2b in inductive step}
and Claim~\ref{verification of 2c and d in induction step}
that if $\delta' > 0$ is sufficiently small, $n$ and $m$ sufficiently large, 
$\mcA, \in \mbX_n(\delta')$, and $\mcA' \in \mbX_m(\delta')$, then
\[
\big|F(q_1(\mcA), \ldots, q_{t(n)}(\mcA)) - F(q_1(\mcA'), \ldots, q_{t(m)}(\mcA'))\big| < \varepsilon.
\]
It follows that there is $\varphi \in [-\mfm, \mfm]$ such that if $n$ is sufficiently large and $\mcA \in \mbX_n(\delta')$,
then $P^\mcA(o_n) = F(q_1(\mcA), \ldots, q_{t(n)}(\mcA)) \in (\varphi - \varepsilon, \varphi + \varepsilon)$.
By Lemma~\ref{stronger induction hypothesis}, there is $c > 0$ such that
\[
\mbbP_n\big(\mbX_n(\delta')\big) \geq 1 - e^{-cn} \quad \text{ for all sufficiently large $n$.}
\]
It follows that 
\[
\mbbP_n\big(\big\{ \mcA \in \mbS_n : P^\mcA(o_n) \in (\varphi - \varepsilon, \varphi + \varepsilon) \big\}\big)
\geq 1 - e^{-cn} \quad \text{ for all sufficiently large $n$.}
\]
This concludes the induction step.

\subsection{Proof of Corollary~\ref{two layers suffice}}

Suppose $\mcB_n$, $n \in \mbbN^+$, is a sequence of NNs subject to the assumptions stated 
in Section~\ref{Model theoretic representation} and
that $\mbbP_n$ is as stated in the same section.
Let $\psi$ be as in Theorem~\ref{main theorem technical statement}.

The proof of Theorem~\ref{main theorem technical statement} 
(see Lemma~\ref{stronger induction hypothesis})
shows that there is $\varphi \in [-\mfm, \mfm]$ such that for all $\delta > 0$ there is $c > 0$ such that
for all sufficiently large $n$
\begin{equation}\label{uniform convergence to varphi}
\mbbP_n\big(\big\{ \mcA \in \mbS_n : 
\text{ for all } a \in L_{n, 2}, P^\mcA(a) \in (\varphi - \delta, \varphi + \delta) \big\}\big)
\geq 1 - e^{-cn}.
\end{equation}
Let $h : [-\mfm, \mfm] \to [-\mfm, \mfm]$ be any continuous function such that $h(\varphi) = \psi$.
Let $F'_2 : [-\mfm, \mfm]^{fin} \to [-\mfm, \mfm]$ be the aggregation function defined by
$F'_2(x_1, \ldots, x_n) = h(F_2(x_1, \ldots, x_n))$ for all $n$ and $(x_1, \ldots, x_n) \in [-\mfm, \mfm]^n$.
By Lemma~\ref{composition of aggregation function and continuous function},
$F'_2$ is a continuous aggregation function.

Define a sequence $\mcB'_n$, $n \in \mbbN^+$, of NNs with only two layers as follows:
\begin{enumerate}
\item Layer 1 of $\mcB'_n$ is identical with layer 1 of $\mcB_n$ and layer 2 of $\mcB'_n$ is $L'_{n, 2} := \{o'_n\}$
where $o'_n$ is an aribtrary choice of neuron in layer 2 of $\mcB_n$.

\item The weights of connections from $L_{n, 1}$ to $o'_n$ in $\mcB'_n$ are the same as the weights of the same connections
in $\mcB_n$.

\item The function $g_2 : [-\mfm, \mfm]^2 \to [-\mfm, \mfm]$ of $\mcB_n$ is also used by $\mcB'_n$,
but instead of $F_2$, $\mcB'_n$ uses the continuous aggregation function $F'_2$, so to determine the
value of $o'_n$, $\mcB'_n$ uses the activation function
\[
F'_2(g_2(v_1, w_1), \ldots, g_2(v_n, w_n))
\]
where $v_i$ is the value of the $i^{th}$ neuron in layer 1 and $w_i$ is the weight of the connection from the 
$i^{th}$ neuron in layer 1 to $o'_n$
\end{enumerate}

\noindent
We let the state space $\mbS'_n$ of $\mcB'_n$ be defined as in 
Section~\ref{Model theoretic representation}
(but with $\mbS'$ and $\mcB'_n$ in place of $\mbS_n$ and $\mcB_n$)
and we let
$\mbbP'_n$ be the probability distribution induced by $\mu_n$ on $\mbS'_n$ in the same way 
as in that section.

Let $\varepsilon > 0$.
As $h$ is continuous and maps $\varphi$ to $\psi$, it follows from~\eqref{uniform convergence to varphi} and
the construction of $\mcB'_n$ that if $\delta > 0$ is chosen sufficiently small 
in~\eqref{uniform convergence to varphi}, 
then for all large enough $n$
\begin{equation*}
\mbbP'_n\big(\big\{ \mcA \in \mbS'_n : 
P^\mcA(o'_n) \in (\psi - \varepsilon, \psi + \varepsilon) \big\}\big)
\geq 1 - e^{-cn}.
\end{equation*}

\begin{rem}\label{remark on generalizations}
{\bf (Generalization to the context of several features with their own values)}
{\rm
Theorem~\ref{main theorem technical statement} and 
Corollary~\ref{two layers suffice}
can without any new ideas, but at the expense of messier book keeping and notation, 
be generalized to the following context.
Let every neuron have $\kappa$ ``features'' where each feature has a value in $[-\mfm, \mfm]$.
The features can be represented by relation symbols $P_1, \ldots, P_\kappa$.
For each layer $i = 2, \ldots, \rho$ and index $j = 1, \ldots, \kappa$ of a feature,
there are continuous aggregation functions $F_{i, j} : [-\mfm, \mfm]^{fin} \to [-\mfm, \mfm]$ and
continuous $g_{i, j} : [-\mfm, \mfm]^{\kappa + 1} \to [-\mfm, \mfm]$ such that the value of the feature $j$ of a neuron $b$ 
of $\mcB_n$ in layer $i$ is equal to 
\[
F_{i, j}\big(g_{i, j}(v_{1, 1}, \ldots, v_{\kappa, 1}, w_1), \ldots, 
g_{i, j}(v_{1, t(n, i-1)}, \ldots, v_{\kappa, t(n, i-1)}, w_{t(n, i-1)})\big)
\]
if $v_{j, 1}, \ldots, v_{j, t(n, i-1)}$ ($j = 1, \ldots, \kappa$) enumerates the values of the feature $j$
of neurons in layer $i-1$
and $w_1, \ldots, w_{t(n, i-1)}$ enumerates the weights of connections from the corresponding neurons in layer $i-1$ to $b$.
In this context we assume that for all $j = 1, \ldots, \kappa$, there is a continuous probability density function
$\mu_j : [-\mfm, \mfm] \to [0, \infty)$ and we replace $\mu_n : [-\mfm, \mfm]^n \to [0, \infty)$ defined as in 
Section~\ref{Model theoretic representation} by the probability density function
$\mu_n : [-\mfm, \mfm]^{\kappa n} \to [0, \infty)$ defined by
\[
\mu_n(x_1, \ldots, x_{\kappa n} = \mu_1(x_1) \cdot \ldots \cdot \mu_1(x_n) \cdot \ldots \cdot
\mu_\kappa(x_{\kappa (n-1) + 1}) \cdot \ldots \cdot \mu_\kappa(x_{\kappa n}).
\]
The version of Theorem~\ref{main theorem technical statement}
in this context states that for every feature $P_j$, $j = 1, \ldots, \kappa$,
there is $\psi_j \in [-\mfm, \mfm]$ such that for all $\varepsilon > 0$ there is $c > 0$ such that for all sufficiently large $n$,
\[
\mbbP_n\big(\big\{ \mcA \in \mbS_n : P_j^\mcA(o_n) \in (\psi_j - \varepsilon, \psi_j + \varepsilon) \big\}\big)
\geq 1 - e^{-cn}.
\]
}\end{rem}

\end{document}